\documentclass{article}
\usepackage{microtype}
\usepackage{graphicx}
\usepackage{subcaption}
\usepackage{booktabs} 
\usepackage{hyperref}

\usepackage[preprint]{icml2026}

\usepackage{amsmath}
\usepackage{amssymb}
\usepackage{mathtools}
\usepackage{amsthm}
\usepackage{xcolor}
\usepackage{multirow}

\usepackage[capitalize,noabbrev]{cleveref}

\theoremstyle{plain}
\newtheorem{theorem}{Theorem}[section]
\newtheorem{proposition}[theorem]{Proposition}

\theoremstyle{definition}

\theoremstyle{remark}

\usepackage{wrapfig}

\usepackage[textsize=tiny]{todonotes}
\icmltitlerunning{NavFormer: IGRF Forecasting in Moving Coordinate Frames}

\begin{document}

\twocolumn[
  \icmltitle{NavFormer: IGRF Forecasting in Moving Coordinate Frames}

  \icmlsetsymbol{equal}{*}

  \begin{icmlauthorlist}
    \icmlauthor{Yoontae Hwang}{pusan,oaq,arrakis}
    \icmlauthor{Dongwoo Lee}{kaist}
    \icmlauthor{Minseok Choi}{oaq,arrakis,kaist} 
    \icmlauthor{Heechan Park}{oaq,arrakis} \\
    \icmlauthor{Yong Sup Ihn}{add}
    \icmlauthor{Daham Kim}{oaq,arrakis,equal}
    \icmlauthor{Deok-Young Lee}{oaq,arrakis,kaist,equal}
  \end{icmlauthorlist}

  \icmlaffiliation{pusan}{Pusan National University, Busan, Republic of Korea}
  \icmlaffiliation{oaq}{OAQ Co. Ltd., Republic of Korea} 
  \icmlaffiliation{add}{Agency for Defense Development, Daejeon, Republic of Korea} 
  \icmlaffiliation{arrakis}{Arrakis Technologies Corp., USA} 
  \icmlaffiliation{kaist}{Korea Advanced Institute of Science and Technology, Daejeon, Republic of Korea}

  \icmlcorrespondingauthor{Deok-Young Lee}{dleeao@oaqcorp.com} 
  \icmlcorrespondingauthor{Daham Kim}{daham.kim@oaqcorp.com}     

  \icmlkeywords{Machine Learning, ICML, IGRF Forecasting}

  \vskip 0.3in
]

\printAffiliationsAndNotice{}  

Triad magnetometer components change with sensor attitude even when the IGRF total intensity target stays invariant. NavFormer forecasts this invariant target with rotation invariant scalar features and a Canonical SPD module that stabilizes the spectrum of window level second moments of the triads without sign discontinuities. The module builds a canonical frame from a Gram matrix per window and applies state dependent spectral scaling in the original coordinates. Experiments across five flights show lower error than strong baselines in standard training, few shot training, and zero shot transfer. The code is available at: https://anonymous.4open.science/r/NavFormer-Robust-IGRF-Forecasting-for-Autonomous-Navigators-0765

\section{Introduction}
Moving coordinate frames turn simple forecasting problems into geometry problems. Sensors on drones, aircraft, and wearables record vectors in a body coordinate system. The platform attitude changes over time. The coordinate axes rotate with the sensor. A fixed physical field then produces different channel readings even when the underlying state is unchanged. The IGRF total intensity is such a scalar \citep{alken2021international}. It is invariant to any rotation of the sensor frame. The input includes vector magnetometer triads whose components depend strongly on orientation \citep{gebre2006calibration}. 
The learning problem must robustly model physical dynamics despite the spectral instability caused by coordinate choice, Also, in a rotating frame, axis aligned channels do not refer to fixed directions. The component at two times can correspond to different physical vectors. This channel drift introduces structural nonstationarity that is not caused by the underlying field. A forecaster then allocates capacity to adapt to varying signal scales instead of learning the field evolution. The effect intensifies when the training set does not cover the test orientation distribution. Few shot training and zero shot transfer then suffer even when the target is invariant.

Existing approaches often try to remove rotation dependence before sequence modeling. One option uses only rotation invariant scalars such as norms and dot products \citep{bulling2014tutorial, canciani2020analysis}. This yields stable inputs but it discards directional information that encodes vector interactions and dynamics. Another option canonicalizes vectors by projecting them into a data derived basis such as a PCA eigenframe \citep{jolliffe2011principal, li2021closer}. Canonicalization can standardize orientation but eigenvectors are defined only up to an independent sign. Small perturbations can flip an eigenvector sign and create a discontinuous representation \citep{davis1970rotation}. Discontinuities break the smooth dependence that gradient based optimization assumes \citep{ionescu2015matrix}. Geometric deep learning instead encodes group structure through equivariant and steerable architectures \citep{cohen2016group,fuchs2020se,satorras2021n}. These methods show that inductive bias for symmetry can improve data efficiency. Our setting still requires a continuous representation that preserves raw vector dynamics for forecasting.

This paper follows a different principle. We do not need strict invariance at the input. We use a window level conditioning step that stabilizes the scale and condition number of the triad second moment. We compute a Gram matrix within each window and apply spectral reweighting in its principal directions. A symmetric positive definite transform provides a continuous geometric preconditioning \citep{huang2017riemannian}. The transform reduces anisotropy of the uncentered second moment and ensures a numerically well-conditioned optimization landscape. This helps the downstream transformer spend less capacity on scale correction. The model still receives directional information and it can exploit cross axis structure.


\section{Related Works}
\subsection{Geometric Constraints in Time Series Forecasting}
Transformer forecasters model long range dependencies with attention and tokenization choices \citep{wen2022transformers}. Patch based and variate based designs such as PatchTST \citep{nie2022time} and iTransformer \citep{liu2023itransformer} improve efficiency and cross channel reasoning, and Crossformer \citep{zhang2023crossformer} further structures attention over channel and temporal axes. Linear baselines such as DLinear \citep{zeng2023transformers} show that simple temporal heads can be competitive when channel semantics are stable. These models usually treat each channel as a fixed semantic entity across windows \citep{chencloser}. Rotating sensing breaks that assumption. For triad measurements, an unknown attitude rotation mixes the three components, so the meaning of each channel shifts while the prediction target can remain approximately rotation invariant \citep{crassidis2005real}. A standard forecaster must learn an implicit coordinate transformation before it can learn temporal dynamics. This makes coordinate system transformation a bottleneck, and it becomes worse under limited training coverage of orientations.

Geometric deep learning builds equivariance into the architecture through group actions, including group equivariant convolution \citep{cohen2016group} and equivariant attention and message passing \citep{fuchs2020se,satorras2021n}. Most of this work targets spatial prediction on sets, graphs, or short horizon dynamics, rather than long horizon forecasting with heterogeneous scalar and vector channels. Magnetic navigation and aeromagnetic compensation historically address rotation and platform effects through calibration models and recursive estimation, often built around Kalman filtering \citep{kalman1960new} and physics based compensation models such as Tolles Lawson \citep{tolles1950magnetic,leliak2009identification}. Modern airborne magnetic anomaly navigation uses these components together with map information and inertial sensing \citep{canciani2020analysis}. NavFormer targets the same bottleneck but replaces an explicit attitude and compensation stack with a learned forecaster that injects a geometric inductive bias at the representation level.

\subsection{Stable Canonicalization via SPD Manifold}
A common route to rotation robustness is canonicalization. One computes a window level second moment, extracts an eigenframe, and projects vectors into that frame. This can reduce nuisance rotations, but it can also introduce discontinuities. Eigenvectors have a sign ambiguity, and near degenerate spectra cause axis instability, so small input perturbations can induce discrete flips or basis drift \citep{davis1970rotation,yu2015useful}. Riemannian learning on the SPD cone provides tools to model second order statistics without leaving the manifold \citep{arsigny2007geometric}. SPDNet and later normalization layers define architectures that preserve positive definiteness and respect SPD geometry \citep{huang2017riemannian, bronstein2021geometric}. Prior work often treats SPD matrices as the primary feature type for classification or recognition \citep{tuzel2006region}.
NavFormer uses SPD structure as a learnable preconditioner for vector time series rather than as an endpoint representation. It extracts a canonical eigenframe from a triad Gram matrix, then applies a state conditioned SPD transform that performs spectral reweighting in that eigenframe while acting in the original coordinates. This avoids the sign flip discontinuity that appears in explicit projections, since the induced SPD map is invariant to eigenvector sign choices. The result can reduce anisotropy of the window second moment and can make the condition number and axis scale more consistent across windows. This can give the downstream transformer an input space that is easier to optimize for forecasting.
\section{NavFormer}\label{section:method}
We forecast IGRF total intensity from sensor data \citep{gnadt2023data, gnadt2023signal}. The primary challenge involves the variability of the sensor's coordinate frame throughout the recording. While the underlying magnetic field magnitude is independent of the sensor's orientation, the raw observations are heavily dependent on it. To address this, we propose an architecture designed with a geometric inductive bias \citep{cohen2016group}. We augment the input with analytically rotation-invariant scalars and apply a geometrically regularized vector modulation layer \citep{fuchs2020se, satorras2021n}. This layer applies SPD based spectral reweighting of the triads in a data driven canonical frame, using the global state summary. Unlike rigid projections that introduce discontinuities our approach aligns the signal statistics while maintaining numerical stability. The resulting spectrally reweighted representation is processed by a grid transformer for sequence modeling.

\begin{figure*}[ht]
\small{
  \centering
  \includegraphics[width=0.83\linewidth]{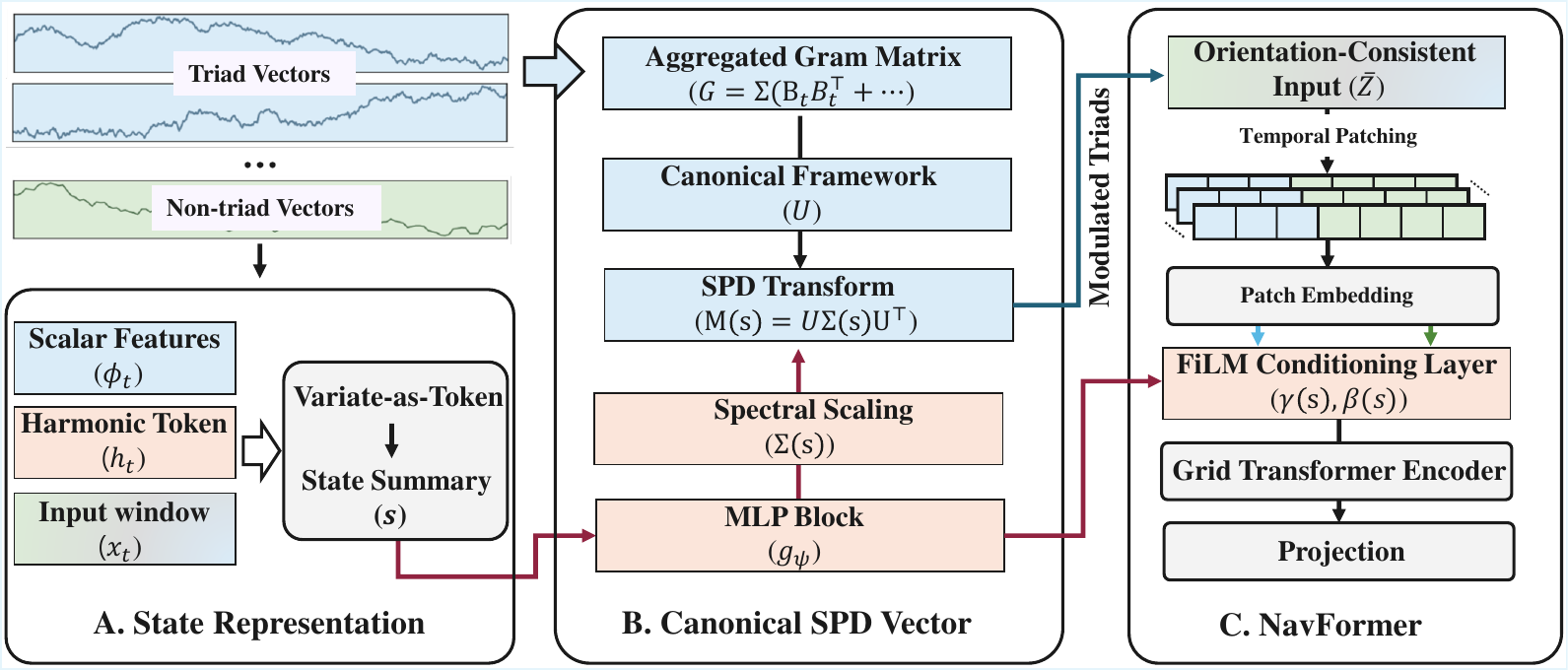}
  \caption{Overview of the NavFormer architecture.}
  \label{figure_main}
 \vspace{-14pt}}
\end{figure*}

\subsection{Problem Setting} \label{subsection:problem}
Let $X \in \mathbb{R}^{L \times D_0}$ be an input window of length $L$ with $D_0$ channels. These input features are heterogeneous. At each time step $t$, the slice is $x_{t} = (B_{t}, C_{t}, D_{t}, r_{t})$, where $B_{t}, C_{t}, D_{t} \in \mathbb{R}^3$ are the triad vectors and $r_{t}$ collects scalar telemetry. The IGRF magnitude is our prediction target and remains invariant under changes in the sensor coordinate frame. We model a change in the sensor coordinate frame as a rotation $R \in \mathbf{SO}(3)$. The attitude change is small, Within each input window, we approximate this change as fixed rotation $R$, which captures piecewise constant orientation over the window. This rotation acts on the triad vectors $(B_{t}, C_{t}, D_{t})$ while leaving the non triad variables $r_{t}$ unchanged $(B_{t}, C_{t}, D_{t}, r_{t}) \mapsto (R B_{t}, R C_{t}, R D_{t}, r_{t})$.

Our goal is to forecast an IGRF total intensity $y \in \mathbb{R}^{L_{\text{pred}} \times 1}$ over a horizon $L_{\text{pred}}$ given the input $X$. A naive model must learn to internalize the rotation group structure entirely from data to map the rotating $X$ to the invariant $y$. Instead of enforcing strict invariance we use a geometrically adaptive conditioning step. The step rescales the triad second moment within each window to reduce anisotropy and stabilize the condition number seen by the sequence model. This helps the model learn temporal dynamics in a feature space with more consistent axis scaling, while preserving directional information.

\subsection{Scalar Geometric Features} \label{subsection:scalar_features}
From the vector triads $B_{t},C_{t},D_{t} \in \mathbb{R}^3$, we construct scalar features that are naturally insensitive to orientation. For three vectors, we define the per-time-step feature
\begin{equation} \label{eq3}
    \phi_{t} \coloneqq \big[ \mathbf{n}_t^\top, \; \mathbf{d}_t^\top, \; \mathbf{c}_t^\top \big]^\top \in \mathbb{R}^9,
\end{equation}
Where $\mathbf{n}_t = [\|B_t\|, \|C_t\|, \|D_t\|]^\top$ contains the norms, $\mathbf{d}_t$ contains the pairwise dot products, and $\mathbf{c}_t$ contains the cross product norms. Equation ~(\ref{eq3}) is analytically invariant under joint rotations of $(B_{t},C_{t},D_{t})$. To capture temporal dynamics, we encode time using a fixed bank of sinusoidal harmonic tokens, motivated by Fourier feature mappings \citep{tancik2020fourier}. For sampling frequency $f_{s}$ and predefined frequencies $\{f_k\}_{k=1}^K$, we let $\tau_{t} \coloneqq t / f_{s}$ and define 
\begin{equation}
    h_{t} \coloneqq \Big[ \sin(2\pi f_k \tau_{t}), \; \cos(2\pi f_k \tau_{t}) \Big]_{k=1}^{K} \in \mathbb{R}^{2K}.
\end{equation}
The concatenated per-step feature thus reads $u_{t} \coloneqq [ x_{t};  \phi_{t};  h_{t} ] \in \mathbb{R}^{D_0 + 9 + 2K}.$. We stack $(u_{t})_{t=1}^L$ over time yields the augmented sequence 
\begin{equation}
    Z \coloneqq [u_1, \dots, u_L]^\top \in \mathbb{R}^{L \times D_{\text{aug}}},
\end{equation}
Where $D_{\text{aug}} \coloneqq D_0 + 9 + 2K$. Next, to capture the global temporal dynamics of each feature while isolating rotation-invariant information, we adopt a variate-as-token strategy \citep{liu2023itransformer}. Therefore, we treat each variable channel as an independent token. Specifically, we transpose the augmented sequence $Z$ and apply a learnable linear projection $W \in \mathbb{R}^{d \times L}$ along the temporal dimension
\begin{equation}
    e_c = W (Z_{:, c}) + b \in \mathbb{R}^{d}, \quad c = 1, \dots, D_{\text{aug}}.
\end{equation}
This operation compresses the time series of the $c$-th variable into a single $d$-dimensional embedding $e_c$. Building on this representation, we construct a global context vector designed to remain robust against sensor orientation. We divide the feature indices $\{1, \dots, D_{\text{aug}}\}$ into two distinct sets $\mathcal{M}$ and $\mathcal{S}$. $\mathcal{M}$ holds the raw vector channels subject to rotation. $\mathcal{S}$ contains the rotation-invariant channels, including $r_{t}$, $\phi_{t}$, and $h_{t}$. We then compute the state summary $\mathbf{s}$ by averaging the embeddings of the invariant set
\begin{equation}
    \mathbf{s} \coloneqq \frac{1}{|\mathcal{S}|} \sum_{c \in \mathcal{S}} e_c \in \mathbb{R}^d.
\end{equation}
Because every feature within $\mathcal{S}$ is immune to the rotation action, the resulting summary $\mathbf{s}$ provides a comparatively stable, frame-independent anchor. Conditioning the geometric modulation layers on $\mathbf{s}$ helps the model adapt its frame-dependent processing using a state signal that remains consistent across sensor orientations.

\subsection{Canonical SPD module} \label{subsection:SPD_FiLM}
We introduce the Canonical SPD module as a geometric preconditioner. While the scalar features $\phi_{t}$ provide a rotation-invariant anchor, the vector channels retain directional information essential for physical consistency. It rescales the window aggregated uncentered 2nd moment of the triad vectors. The modulation depends on the global state summary. The output triads transform by the same rotation as the input triads. Functionally, the process is twofold. First, we extract a canonical frame from an aggregated triad Gram matrix. Second, we apply a state dependent spectral reweighting within that frame using an SPD transform \citep{huang2017riemannian, li2018towards}. A standard approach to geometric learning involves projecting input vectors onto a canonical basis $U$ (e.g., computing $U^\top B_t$) to achieve invariance. However, eigenvectors $U$ are only defined up to a sign flip ($u \leftrightarrow -u$). This sign ambiguity creates discontinuity in the projected features, as small perturbations in the input can cause discrete jumps in the representation. Such instability hinders the optimization of continuous forecasting models. To resolve this, we employ a spectral modulation $M(\mathbf{s}) = U \Sigma(\mathbf{s}) U^\top$. This transform operates in the global frame and is invariant to the sign choice of $U$, as $(-U)\Sigma (-U)^\top = U \Sigma U^\top$. Therefore this component acts as a geometric preconditioner.It rescales energy along the principal axes of the aggregated uncentered second moment. This removes the sign flip ambiguity of eigenvectors. This can reduce anisotropy in the second moment.

\paragraph{Canonical Frame from Aggregated Gram Matrix}
To establish a stable geometric reference, we compute an aggregated Gram matrix for a window of triad vectors $\{B_{t}, C_{t}, D_{t}\}_{t=1}^L$. This matrix is an uncentered 2nd moment statistic of the triad vectors over the window.
\begin{equation}
G \coloneqq \sum_{t=1}^L ( B_{t} B_{t}^\top + C_{t} C_{t}^\top + D_{t} D_{t}^\top ) \in \mathbb{R}^{3 \times 3}.
\label{eq:eigen}
\end{equation}
By definition, $G$ is symmetric and positive semi-definite. We perform an eigen-decomposition $ G = U \Lambda U^\top$, yielding eigenvalues $\lambda_1 \ge \lambda_2 \ge \lambda_3 \ge 0$ and orthonormal eigenvectors $U = [u_{1} u_{2}u_{3}] \in \mathbb{R}^{3 \times3}$. We designate $U$ as the canonical frame, as it intrinsically maps the window's dominant directions of variation.To ensure that the canonical frame is numerically stable, we use a mild separation assumption on the aggregated Gram matrix $G$. Specifically, for the vast majority of windows, we assume $G$ has three distinct eigenvalues with a non-trivial spectral gap: there exists $\delta>0$ such that
\begin{equation}
\frac{\lambda_1 - \lambda_2}{\lambda_1 + \epsilon} \ge \delta, \qquad \frac{\lambda_2 - \lambda_3}{\lambda_2 + \epsilon} \ge \delta.
\end{equation}
Under this assumption, the eigenbasis $U$ is unique up to independent sign flips and varies smoothly with $G$, which is crucial for continuity of the canonicalization map.

\paragraph{State-Dependent SPD Scaling} 
 We map the state summary $\mathbf{s} \in \mathbb{R}^d$ (defined in Section \ref{subsection:scalar_features}) to three scalar scales via an MLP
\begin{equation}
\mathbf{d}(\mathbf{s}) = (d_{1}, d_{2}, d_{3}) = g_\psi(\mathbf{s})\in \mathbb{R}^3,
\end{equation}
where $g_{\psi}$ is a lightweight neural network. Each scalar $d_{i}$ targets one of the three canonical eigenvectors ($u_{1}$, $u_{2}$, $u_{3}$) derived from the aggregated Gram matrix, allowing the model to independently modulate signal energy along each principal axis. To reduce excessive contraction along any canonical axis, we impose a lower bound of $\epsilon$ on the scales
\begin{equation}
\sigma_{i}(\mathbf{s}) \coloneqq \epsilon + \text{softplus}(d_i), \quad i = 1,2,3.
\end{equation}
These scales form a diagonal matrix
\begin{equation}
\Sigma(\mathbf{s}) \coloneqq \mathbf{diag}(
\sigma_1(\mathbf{s}),
\sigma_2(\mathbf{s}),
\sigma_3(\mathbf{s}))
\in \mathbb{R}^{3 \times 3}.
\end{equation}

We use it as a safeguard against near zero scaling that would suppress a direction of the triad signal. The model can still adjust the relative weighting across canonical directions through $\Sigma(\mathbf{s})$. We then define the resulting canonical SPD transform as
\begin{equation}
M(\mathbf{s}) \coloneqq U \Sigma(\mathbf{s}) U^\top \in \mathbf{SPD}(3),
\end{equation}
where $\mathbf{SPD}(3)$ is the cone of $3\times 3$ symmetric positive definite matrices. Applying $M(\mathbf{s})$ to the input triads yields the modulated sequences
\begin{equation}
\tilde{B}_{t} \coloneqq M(\mathbf{s}) B_{t},\quad
\tilde{C}_{t} \coloneqq M(\mathbf{s}) C_{t},\quad
\tilde{D}_{t} \coloneqq M(\mathbf{s}) D_{t}.
\end{equation}
The following proposition confirms that $M(\mathbf{s})$ preserves the canonical eigenvectors while selectively reweighting the covariance spectrum.

\begin{proposition}[Spectral Alignment of Canonical SPD] \label{proposition:spectral_alignment}
Let $G = U \Lambda U^\top$ be the aggregated Gram matrix and $M(\mathbf{s}) =U \Sigma(\mathbf{s}) U^\top$ the canonical SPD transform. Define the Gram matrix of the modulated triads as
\begin{equation}
\tilde{G}(\mathbf{s})\coloneqq\sum_{t=1}^L ( \tilde{B}_{t} \tilde{B}_{t}^\top + \tilde{C}_{t} \tilde{C}_{t}^\top + \tilde{D}_{t} \tilde{D}_{t}^\top ).
\end{equation}

Then \begin{equation}
\tilde{G}(\mathbf{s}) = U \Sigma(\mathbf{s}) \Lambda \Sigma(\mathbf{s}) U^\top.
\end{equation}

In particular, $\tilde{G}(\mathbf{s})$ has eigenvectors $U$ and eigenvalues$\sigma_i(\mathbf{s})^2 \lambda_i$ for $i=1,2,3$.\end{proposition}

\begin{proof}[Proof] See Appendix ~\ref{appendix_math}. \end{proof}

Proposition \ref{proposition:spectral_alignment} demonstrates that the Canonical SPD reweights the spectrum of the triad Gram matrix. While the output vectors $\tilde{B}_{t}$ remain rotation-equivariant, their statistical structure is stabilized. This spectral normalization acts as a geometric preconditioner, providing the downstream forecaster with a numerically well-conditioned input space that is easier to optimize. Their window aggregated uncentered 2nd moment becomes $\tilde{G}(\mathbf{s}) = U \Sigma(\mathbf{s}) \Lambda \Sigma(\mathbf{s}) U^\top$. This can reduce variation in the spectrum of the second moment across windows. The downstream sequence model then receives inputs with a more consistent spectrum. This can improve numerical conditioning during training.

\subsection{Patch-Channel Grid Transformer}\label{subsection:patch_channel_grid} 
The Canonical SPD module and invariant scalar augmentation yield a sequence that is geometrically regularized, serving as the input to our patch-channel grid transformer. We define the feature slice $\bar{x}_{t} \in \mathbb{R}^{D_0}$ by substituting the raw magnetometer triads with their SPD-modulated counterparts $(\tilde{B}_{t}, \tilde{C}_{t}, \tilde{D}_{t})$. All other channels, including scalar telemetry, are preserved in their original form. The invariant scalars $\phi_{t}$ are derived from the raw triads $(B_{t}, C_{t}, D_{t})$ before any SPD modulation occurs to preserve the original signal intensity information. We assemble the final spectrally reweighted vector for each step as 
\begin{equation}
\bar{u}_{t} \coloneqq [\bar{x}_{t};\phi_{t};h_{t}] \in \mathbb{R}^{D_{\text{aug}}}.
\end{equation}
where $D_{\text{aug}} = D_0 + 9 + 2K$. These are stacked temporally to form the input tensor as here:
\begin{equation}
    \bar{Z} \coloneqq [\bar{u}_1, \dots, \bar{u}_L]^\top \in \mathbb{R}^{L \times D_{\text{aug}}}.
\end{equation}
To efficiently resolve both temporal dependencies and cross-channel interactions \citep{zhang2023crossformer} within $\bar{Z}$, we utilize a grid based self-attention mechanism operating on channel-specific patches.

\paragraph{Temporal Patching per Channel}  We process each channel independently by extracting overlapping temporal patches \citep{nie2022time}. Let $P$ denote the patch length and $S$ the stride. To guarantee full coverage at the sequence boundaries, we apply replication padding of $S$ time steps to the end of the sequence, resulting in an effective length of $L+S$. The number of patches is then $M_{\mathbf{p}} \coloneqq \lfloor \frac{L - P}{S} \rfloor + 2.$

Working from a channel-first perspective $\bar{Z}^\top \in \mathbb{R}^{D_{\text{aug}} \times L}$, we define the temporal patch vector for channel $c \in \{1,\dots,D_{\text{aug}}\}$ and patch index $m \in \{1,\dots,M_{\mathbf{p}}\}$
\begin{equation}
z_{c,m} \coloneqq
[\bar{Z}^\top_{c,(m-1)S+1},\dots,\bar{Z}^\top_{c,(m-1)S+P}]^\top
\in \mathbb{R}^{P},
\end{equation}
where any indices exceeding the original length $L$ are populated via the replication padding.

\paragraph{Patch Embedding} We map each raw patch into the model dimension $d$ via a shared linear projection as here
\begin{equation}
p_{c,m} \coloneqq W_{\text{patch}} z_{c,m} + b_{\text{patch}}
\in \mathbb{R}^{d},
\end{equation}

with weights $W_{\text{patch}} \in \mathbb{R}^{d \times P}$ and bias $b_{\text{patch}}\in\mathbb{R}^{d}$. This operation yields a structured 2D grid of tokens $\{p_{c,m}\}$, indexed by channel $c$ and patch $m$. These embeddings encode local temporal patterns, forming the substrate for the global state-conditioning mechanism described next.

\paragraph{FiLM Conditioning of Token Embeddings}\label{subsubsection:film_conditioning}  To ground the local embeddings in the global system state, we modulate the patch tokens $p_{c,m}$ using the summary vector $\mathbf{s}\in\mathbb{R}^{d}$ (derived in Sec.~\ref{subsection:scalar_features}). We compute the affine FiLM parameters via a projection network $q_\omega$ as
\begin{equation}
\begin{split}
    (\hat{\gamma}(\mathbf{s}), \beta(\mathbf{s})) &= q_\omega(\mathbf{s}) \in \mathbb{R}^{2d}, \\
    \gamma(\mathbf{s}) &\coloneqq 1 + \tanh(\hat{\gamma}(\mathbf{s})).
\end{split}
\end{equation}
This modulation is applied selectively. We apply it only to the magnetometer channels so vector token processing depends on the global state summary. We set $\tilde{p}_{c,m} \coloneqq \gamma(\mathbf{s}) \odot p_{c,m} + \beta(\mathbf{s})$ with $c \in \mathcal{M}$. Auxiliary telemetry channels remain unchanged, so $\tilde{p}_{c,m} \coloneqq p_{c,m}$ with $c \notin \mathcal{M}$. This mechanism diffuses global state context across the token grid. The tanh bounded scaling factor $\gamma(\mathbf{s})$ helps keep the modulation stable. Finally, we add learnable embeddings for channel index and patch index, then flatten the grid into a unified token sequence of length $D_{\text{aug}}M_{\mathbf{p}}$
\begin{equation}
T \coloneqq \mathbf{vec}(\{\tilde{p}_{c,m}\}_{c=1,m=1}^{D_{\text{aug}},M_{\mathbf{p}}})
\in \mathbb{R}^{(D_{\text{aug}} M_{\mathbf{p}})\times d},
\end{equation}
This structure enables the transformer encoder to jointly resolve dependencies across both temporal patches and distinct sensor channels.

\paragraph{Grid Self-Attention Encoder}
We apply a Transformer encoder \citep{vaswani2017attention} to the flattened token sequence
\begin{equation}
H \coloneqq \text{Encoder}(T) \in \mathbb{R}^{(D_{\text{aug}} M_{\mathbf{p}})\times d}.
\end{equation}
Mapping the output back to the original grid structure recovers a representation $H_{c,m}\in\mathbb{R}^{d}$ specific to each channel and patch. Since our self-attention spans the full grid, every token effectively possesses a global receptive field. This allows the model to aggregate context from any channel and any temporal patch simultaneously. The result is a architecture that captures rich cross-sensor interactions and long-range temporal reasoning, all while maintaining computational efficiency by operating on a reduced temporal resolution ($M_{\mathbf{p}} \ll L$).

\subsection{Channel-Shared Horizon Projection}
We adopt a channel shared decoding strategy to synthesize the forecast from the grid representation. The encoder output $H$ preserves the distinct identity of each variable while encoding their temporal dependencies. To generate the prediction we first flatten the patch dimension for each channel independently. This operation collapses the local temporal tokens into a unified channel vector
\begin{equation}
v_c = \mathbf{vec}([H_{c,1}, \dots, H_{c,M_{\mathbf{p}}}]) \in \mathbb{R}^{M_{\mathbf{p}}d}
\end{equation}
where $c$ indexes both the raw sensor inputs and the geometric invariants. We then apply a linear projection $W_{\text{head}} \in \mathbb{R}^{L_{\text{pred}} \times (M_{\mathbf{p}}d)}$ to map these latent temporal features directly to the forecasting horizon
\begin{equation}
\hat{y}_c = W_{\text{head}} v_c + b_{\text{head}} \in \mathbb{R}^{L_{\text{pred}} \times 1}
\end{equation}
Crucially we share $W_{\text{head}}$ and $b_{\text{head}}$ across all $D_{\text{aug}}$ channels. This design leverages the variate based tokenization philosophy where the model learns universal temporal dynamics applicable to both magnetic field components and scalar telemetry. It acts as a regularization mechanism that prevents overfitting to channel specific noise.

\begin{table*}[!htbp]
    \centering
{\fontsize{6.0}{8.0}\selectfont
\begin{tabular}{ccccccccccccccccc}
\hline
\multicolumn{17}{l}{\textbf{Panel A.} Lookback window $L=30$ (Avg. over prediction lengths $L_{\text{pred}} \in$ \{60, 120\})}                                                                                                                                                                                                                                                                                                                                                                                                                      \\
Methods                       & \multicolumn{2}{c}{NavFormer}                                                   & \multicolumn{2}{c}{PatchTST}                                                    & \multicolumn{2}{c}{PAttn}                                & \multicolumn{2}{c}{TimeFilter} & \multicolumn{2}{c}{iTransformer}                         & \multicolumn{2}{c}{TimesNet} & \multicolumn{2}{c}{DLinear} & \multicolumn{2}{c}{LTC-CFC} \\ \hline
Metric                        & MAE                                    & RMSE                                   & MAE                                    & RMSE                                   & MAE                                    & RMSE            & MAE                & RMSE      & MAE             & RMSE                                   & MAE               & RMSE     & MAE          & RMSE         & MAE          & RMSE        \\ \hline
NV2                           & {\color[HTML]{FE0000} \textbf{0.0210}} & {\color[HTML]{FE0000} \textbf{0.0595}} & \textbf{0.0226}                        & {\color[HTML]{3531FF} \textbf{0.0632}} & 0.0225                                 & \textbf{0.0636} & 0.0225             & 0.0657    & 0.0237          & 0.0661                                 & 0.0287            & 0.0747   & 0.0737       & 0.1194       & 0.2883       & 0.3769      \\
NV3                           & {\color[HTML]{3531FF} \textbf{0.0351}} & {\color[HTML]{3531FF} \textbf{0.0802}} & {\color[HTML]{FE0000} \textbf{0.0349}} & {\color[HTML]{FE0000} \textbf{0.0792}} & \textbf{0.0369}                        & \textbf{0.0816} & 0.0403             & 0.0965    & 0.0380          & 0.0826                                 & 0.0486            & 0.0944   & 0.1125       & 0.1637       & 0.2321       & 0.2974      \\
NV4                           & {\color[HTML]{FE0000} \textbf{0.0238}} & {\color[HTML]{FE0000} \textbf{0.0774}} & {\color[HTML]{3531FF} \textbf{0.0261}} & {\color[HTML]{3531FF} \textbf{0.0798}} & {\color[HTML]{000000} \textbf{0.0354}} & \textbf{0.0908} & 0.0439             & 0.1076    & 0.0516          & 0.1060                                 & 0.0607            & 0.1214   & 0.1299       & 0.1825       & 0.3507       & 0.4428      \\
NV6                           & {\color[HTML]{FE0000} \textbf{0.1216}} & {\color[HTML]{FE0000} \textbf{0.2683}} & {\color[HTML]{3531FF} \textbf{0.1336}} & {\color[HTML]{3531FF} \textbf{0.2737}} & \textbf{0.1480}                        & \textbf{0.2950} & 0.1488             & 0.2970    & 0.1592          & 0.3061                                 & 0.1913            & 0.3492   & 0.3145       & 0.4871       & 0.3591       & 0.4387      \\
NV7                           & {\color[HTML]{FE0000} \textbf{0.0520}} & {\color[HTML]{FE0000} \textbf{0.1005}} & {\color[HTML]{3531FF} \textbf{0.0532}} & \textbf{0.1040}                        & 0.0575                                 & 0.1070          & 0.0607             & 0.1181    & 0.0670          & 0.1135                                 & \textbf{0.0567}   & 0.1025   & 0.1238       & 0.1727       & 0.3526       & 0.4256      \\ \hline
\multicolumn{1}{l}{Avg. Rank} & 1.3                                    & 1.4                                    & 2.2                                    & 2.0                                    & 3.1                                    & 3.0             & 3.9                & 3.9       & 4.5             & 3.9                                    & 6.0               & 6.0      & 7.0          & 7.0          & 8.0          & 8.0         \\ \hline
\multicolumn{1}{l}{}          &                                        &                                        &                                        &                                        &                                        &                 &                    &           &                 &                                        &                   &          &              &              &              &             \\
\multicolumn{1}{l}{}          &                                        &                                        &                                        &                                        &                                        &                 &                    &           &                 &                                        &                   &          &              &              &              &             \\ \hline
\multicolumn{17}{l}{\textbf{Panel B.} Lookback window $L=60$ (Avg. over prediction lengths $L_{\text{pred}} \in $ \{60, 120\})}                                                                                                                                                                                                                                                                                                                                                                                                                      \\
Methods                       & \multicolumn{2}{c}{NavFormer}                                                   & \multicolumn{2}{c}{PatchTST}                                                    & \multicolumn{2}{c}{PAttn}                                & \multicolumn{2}{c}{TimeFilter} & \multicolumn{2}{c}{iTransformer}                         & \multicolumn{2}{c}{TimesNet} & \multicolumn{2}{c}{DLinear} & \multicolumn{2}{c}{LTC-CFC} \\ \hline
Metric                        & MAE                                    & RMSE                                   & MAE                                    & RMSE                                   & MAE                                    & RMSE            & MAE                & RMSE      & MAE             & RMSE                                   & MAE               & RMSE     & MAE          & RMSE         & MAE          & RMSE        \\ \hline
NV2                           & {\color[HTML]{FE0000} \textbf{0.0215}} & {\color[HTML]{FE0000} \textbf{0.0637}} & 0.0252                                 & \textbf{0.0658}                        & 0.0228                                 & 0.0647          & 0.0239             & 0.0683    & \textbf{0.0233} & 0.0681                                 & 0.0315            & 0.0773   & 0.0608       & 0.1025       & 0.2848       & 0.3783      \\
NV3                           & {\color[HTML]{FE0000} \textbf{0.0404}} & \textbf{0.0916}                        & {\color[HTML]{3531FF} \textbf{0.0408}} & {\color[HTML]{FE0000} \textbf{0.0869}} & \textbf{0.0409}                        & \textbf{0.0916} & 0.0442             & 0.0996    & 0.0412          & {\color[HTML]{3531FF} \textbf{0.0894}} & 0.0522            & 0.1021   & 0.1044       & 0.1544       & 0.2231       & 0.2905      \\
NV4                           & {\color[HTML]{FE0000} \textbf{0.0251}} & {\color[HTML]{3531FF} \textbf{0.0836}} & {\color[HTML]{3531FF} \textbf{0.0287}} & {\color[HTML]{FE0000} \textbf{0.0804}} & 0.0327                                 & \textbf{0.0881} & \textbf{0.0307}    & 0.0891    & 0.0462          & 0.1015                                 & 0.0586            & 0.1226   & 0.1144       & 0.1668       & 0.3528       & 0.4446      \\
NV6                           & {\color[HTML]{FE0000} \textbf{0.1365}} & {\color[HTML]{3531FF} \textbf{0.2889}} & {\color[HTML]{3531FF} \textbf{0.1366}} & {\color[HTML]{FE0000} \textbf{0.2807}} & \textbf{0.1398}                        & \textbf{0.2948} & 0.1450             & 0.2961    & 0.1602          & 0.3130                                 & 0.2036            & 0.3753   & 0.3551       & 0.5404       & 0.3591       & 0.4383      \\
NV7                           & {\color[HTML]{3531FF} \textbf{0.0578}} & {\color[HTML]{3531FF} \textbf{0.1112}} & {\color[HTML]{FE0000} \textbf{0.0574}} & \textbf{0.1121}                        & 0.0633                                 & 0.1190          & \textbf{0.0611}    & 0.1166    & \textbf{0.0611} & {\color[HTML]{FE0000} \textbf{0.1098}} & 0.0782            & 0.1333   & 0.1047       & 0.1537       & 0.3565       & 0.4323      \\ \hline
Avg. Rank                     & 1.5                                    & 2.5                                    & 2.7                                    & 1.9                                    & 3.4                                    & 2.8             & 4.6                & 4.2       & 3.8             & 3.7                                    & 5.8               & 5.9      & 7.0          & 7.0          & 8.0          & 8.0         \\ \hline
\end{tabular}}
        \caption{Forecasting performance comparison on the IGRF dataset. We report the Mean Absolute Error (MAE) and Root Mean Squared Error (RMSE) averaged over prediction lengths $T \in \{60, 120\}$. \textbf{Panel A} presents results with a lookback window of $L=30$, while \textbf{Panel B} uses a lookback window of $L=60$. The best results are highlighted in \textcolor[HTML]{FE0000}{red}, the second-best in \textcolor[HTML]{3531FF}{blue} and third-best in \textbf{bold}, respectively. Full results for each specific prediction length are provided in the Appendix ~\ref{app:full_results}.}
    \label{main_results}
    \vspace{-18pt}
\end{table*}

\section{Experiment} 
We evaluate the proposed NavFormer on the task of forecasting the IGRF total intensity from rotating vector magnetometer data. We compare the method against state of the art time series forecasting models. We analyze performance in standard, few shot, and zero shot settings. We also analyze the contribution of the geometric components through ablation and spectral stability checks.

\subsection{Implementation details}
\paragraph{Dataset} The magnetometer dataset\citep{gnadt2020signal} contains multiple flights of a fixed wing aircraft. The platform carries several vector fluxgate sensors and scalar magnetometers that record the local geomagnetic field together with auxiliary telemetry.  Scalar magnetometer is mounted at the end of the tail stinger and serves as a reference IGRF total intensity. Flights NV2, NV3, NV4, NV6, and NV7 cover compensation maneuvers, free flight segments, and regional survey lines over Eastern Ontario and nearby regions. The raw recordings form multivariate time series of vector triads and scalar channels with long temporal context. For each flight we split the time axis into contiguous train, validation, and test blocks with a 60\%, 20\%, and 20\%, then form windows within each block. Details of the dataset description provided in Appendix ~\ref{app:dataset}.

\paragraph{Baseline Models} We compare NavFormer with patch based and attention based forecasters that represent strong modern sequence modeling baselines. We include PatchTST \citep{nie2022time}, iTransformer \citep{liu2023itransformer}, and PAttn \citep{tan2024language}. We also include TimesNet \citep{wu2022timesnet} and the linear model DLinear \citep{zeng2023transformers}. We report results for two additional baselines, TimeFilter \citep{hu2025timefilter} and LTC-CFC \citep{nerrise2024physics}. For each baseline, we tune hyperparameters on the validation split and select the checkpoint with the lowest validation error before reporting test performance. Details of the parameter search space are provided in Appendix ~\ref{appendix:hyper}.

\paragraph{Task and metrics} Models take an input window of length $L \in \{30,60\}$ and forecast the IGRF total intensity over horizons $L_{\text{pred}} \in \{60,120\}$.
We report mean absolute error and root mean squared error. We average results over the two horizons in the main tables.


\subsection{Main forecasting performance}
\paragraph{Results} Table~\ref{main_results} compares NavFormer with baselines under two context lengths. With $L=30$, NavFormer attains the lowest MAE on four out of five flights and the lowest RMSE on four out of five flights. The MAE gains over the strongest baseline reach $8.8\%$ on NV4 and $9.0\%$ on NV6, relative to PatchTST. NV4 and NV6 include stronger attitude variation, which induces pronounced channel drift when triads are expressed in a rotating sensor frame. NavFormer reduces this drift by using a Gram defined canonical frame and a state conditioned SPD spectral scaling that stabilizes triad statistics across orientations. With $L=60$, NavFormer again achieves the lowest MAE on four out of five flights. The benefit strengthens on NV4 where NavFormer improves MAE by $12.5\%$ relative to PatchTST. These results match the method design. The invariant scalar pathway provides a stable summary of the physical state, while the canonical SPD modulation and FiLM conditioning correct frame dependent variation in the triads and patch tokens. In contrast, baselines process raw triad components as fixed channels. They must infer rotation robustness from data, which requires coverage over sensor orientations and can waste model capacity on coordinate alignment. The poor performance of DLinear and LTC-CFC in Table~\ref{main_results} is consistent with this failure mode, since these models have limited ability to model nonlinear interactions among rotating vector channels.

\subsection{Generalization under limited or shifted supervision}
\paragraph{Setting} We evaluate generalization under two protocols. In the few shot setting, we train each model on only $5\%$ of the training windows from the target flight while keeping the same validation and test splits. In the zero shot transfer setting, we train on one source flight and evaluate on the remaining flights without fine tuning.

\subsubsection{Few shot forecasting}
\paragraph{Results} Table~\ref{main_fewshot} reports performance when training data is limited. NavFormer attains the lowest MAE on every flight and under both lookback lengths. With $L=30$, NavFormer reduces the average MAE from $0.0996$ for PatchTST to $0.0843$, which corresponds to a $15.38\%$ reduction. With $L=60$, the reduction increases from $0.0987$ to $0.0809$, which corresponds to $18.05\%$. The largest gap occurs on NV4. For $L=60$, NavFormer attains MAE $0.0484$, while PatchTST attains $0.0805$, which corresponds to a $39.88\%$ reduction. These results suggest improved sample efficiency. The rotation invariant scalars provide supervision signals that remain stable across sensor frames. The canonical SPD modulation then maps the vector triads to a representation that varies less with orientation. Baselines must learn this normalization from limited examples, which increases error when the training subset does not cover the orientation distribution of the test windows.

\begin{table}[!htbp]
    \centering
{\fontsize{6.5}{8.0}\selectfont
\begin{tabular}{cccccc}
\hline
\multicolumn{6}{l}{Panel A. Lookback window $L=30$ (Avg. over prediction lengths $L_{\text{pred}} \in$ \{60, 120\})}                                                                                                                                                                  \\ \hline
Methods              & NavFormer                              & PatchTST                               & PAttn                & TimeFilter                             & iTransformer         \\ \hline
NV2                  & {\color[HTML]{FE0000} \textbf{0.0326}} & \textbf{0.0381}                        & 0.0379               & {\color[HTML]{3531FF} \textbf{0.0341}} & 0.0369               \\
NV3                  & {\color[HTML]{FE0000} \textbf{0.0500}} & {\color[HTML]{3531FF} \textbf{0.0554}} & \textbf{0.0625}      & 0.0648                                 & 0.0653               \\
NV4                  & {\color[HTML]{FE0000} \textbf{0.0604}} & {\color[HTML]{3531FF} \textbf{0.0782}} & 0.0833               & \textbf{0.0809}                        & 0.0852               \\
NV6                  & {\color[HTML]{FE0000} \textbf{0.2061}} & \textbf{0.2472}                        & 0.2624               & {\color[HTML]{3531FF} \textbf{0.2411}} & 0.2638               \\
NV7                  & {\color[HTML]{FE0000} \textbf{0.0724}} & {\color[HTML]{3531FF} \textbf{0.0792}} & 0.0858               & \textbf{0.0847}                        & 0.0872               \\ \hline
\multicolumn{1}{l}{} & \multicolumn{1}{l}{}                   & \multicolumn{1}{l}{}                   & \multicolumn{1}{l}{} & \multicolumn{1}{l}{}                   & \multicolumn{1}{l}{} \\ \hline
\multicolumn{6}{l}{Panel B. Lookback window $L=60$ (Avg. over prediction lengths $L_{\text{pred}} \in $ \{60, 120\})}                                                                                                                                                                  \\ \hline
Methods              & NavFormer                              & PatchTST                               & PAttn                & TimeFilter                             & iTransformer         \\ \hline
NV2                  & {\color[HTML]{FE0000} \textbf{0.0282}} & \textbf{0.0396}                        & 0.0380               & {\color[HTML]{3166FF} \textbf{0.0362}} & 0.0415               \\
NV3                  & {\color[HTML]{FE0000} \textbf{0.0506}} & {\color[HTML]{3166FF} \textbf{0.0566}} & \textbf{0.0638}      & 0.0711                                 & 0.0745               \\
NV4                  & {\color[HTML]{FE0000} \textbf{0.0484}} & {\color[HTML]{3166FF} \textbf{0.0805}} & \textbf{0.0877}      & 0.0916                                 & 0.1001               \\
NV6                  & {\color[HTML]{FE0000} \textbf{0.1991}} & {\color[HTML]{3166FF} \textbf{0.2353}} & 0.2477               & \textbf{0.2435}                        & 0.3062               \\
NV7                  & {\color[HTML]{FE0000} \textbf{0.0782}} & {\color[HTML]{3166FF} \textbf{0.0816}} & \textbf{0.0852}      & 0.0920                                 & 0.1016               \\ \hline
\end{tabular}}
    \caption{Few shot forecasting. We train each model with a 5\% subset of the training windows of the target flight and we keep the same validation and test splits. The best, second-best, and third-best results for each metric are highlighted in \textcolor[HTML]{FE0000}{red}, \textcolor[HTML]{3531FF}{blue}, and \textbf{bold}, respectively.}
    \label{main_fewshot}
    \vspace{-23pt}
\end{table}
\subsubsection{Zero shot transfer}
\paragraph{Results} Table~\ref{main_zeroshot} measures generalization across flights whose maneuver profiles and orientation statistics differ. NavFormer attains the lowest MAE in nine of ten settings and it achieves the best average MAE in both panels. With $L=30$, the average MAE decreases from $0.0553$ for PatchTST to $0.0530$, which corresponds to a $4.16\%$ reduction. With $L=60$, the average MAE decreases from $0.0603$ for PAttn to $0.0587$, which corresponds to a $2.65\%$ reduction. The exceptions occur on NV4 with $L=30$ and on NV6 with $L=60$, where a baseline attains slightly lower MAE. The advantage persists under this stronger distribution shift because NavFormer encodes the symmetry of the task. The state summary depends on channels that are invariant to $\mathbf{SO}(3)$ rotations, so the conditioning signals remain comparable across flights. The canonical frame derives from the Gram matrix of the triads within each window, so it adapts to orientation changes without requiring flight specific training.

\begin{table}[!htbp]
    \centering
{\fontsize{6.5}{8.0}\selectfont
\begin{tabular}{cccccc}
\hline
\multicolumn{6}{l}{Panel A. Lookback window $L=30$ (Avg. over prediction lengths $L_{\text{pred}} \in$ \{60, 120\})}                                                                                                                                                                                                      \\ \hline
Methods              & NavFormer                              & PatchTST                               & PAttn                                  & TimeFilter                             & iTransformer                           \\ \hline
NV2                  & {\color[HTML]{FE0000} \textbf{0.0608}} & \textbf{0.0690}                        & {\color[HTML]{3531FF} \textbf{0.0638}} & 0.0673                                 & 0.0692                                 \\
NV3                  & {\color[HTML]{FE0000} \textbf{0.0511}} & {\color[HTML]{3531FF} \textbf{0.0540}} & \textbf{0.0541}                        & 0.0559                                 & 0.0599                                 \\
NV4                  & {\color[HTML]{3531FF} \textbf{0.0596}} & {\color[HTML]{FE0000} \textbf{0.0587}} & \textbf{0.0651}                        & 0.0720                                 & 0.0794                                 \\
NV6                  & {\color[HTML]{FE0000} \textbf{0.0396}} & \textbf{0.0403}                        & {\color[HTML]{3531FF} \textbf{0.0399}} & 0.0421                                 & 0.0439                                 \\
NV7                  & {\color[HTML]{FE0000} \textbf{0.0537}} & {\color[HTML]{3531FF} \textbf{0.0543}} & \textbf{0.0569}                        & 0.0597                                 & 0.0636                                 \\ \hline
\multicolumn{1}{l}{} & \multicolumn{1}{l}{}                   & \multicolumn{1}{l}{}                   & \multicolumn{1}{l}{}                   & \multicolumn{1}{l}{}                   & \multicolumn{1}{l}{}                   \\ \hline
\multicolumn{6}{l}{Panel B. Lookback window $L=60$ (Avg. over prediction lengths $L_{\text{pred}} \in $ \{60, 120\})}                                                                                                                                                                                                      \\ \hline
Methods              & NavFormer                              & PatchTST                               & PAttn                                  & TimeFilter                             & iTransformer                           \\ \hline
NV2                  & {\color[HTML]{FE0000} \textbf{0.0647}} & 0.0795                                 & \textbf{0.0676}                        & 0.0745                                 & {\color[HTML]{3531FF} \textbf{0.0696}} \\
NV3                  & {\color[HTML]{FE0000} \textbf{0.0597}} & \textbf{0.0626}                        & 0.0629                                 & 0.0635                                 & {\color[HTML]{3531FF} \textbf{0.0619}} \\
NV4                  & {\color[HTML]{FE0000} \textbf{0.0667}} & 0.0751                                 & {\color[HTML]{3531FF} \textbf{0.0705}} & \textbf{0.0713}                        & 0.0786                                 \\
NV6                  & {\color[HTML]{3531FF} \textbf{0.0434}} & 0.0485                                 & {\color[HTML]{FE0000} \textbf{0.0409}} & {\color[HTML]{3531FF} \textbf{0.0412}} & \textbf{0.0445}                        \\
NV7                  & {\color[HTML]{FE0000} \textbf{0.0588}} & 0.0677                                 & \textbf{0.0596}                        & {\color[HTML]{3531FF} \textbf{0.0593}} & 0.0658                                 \\ \hline
\end{tabular}}
    \caption{Zero shot forecasting. We train on source flights and evaluate on the held out flight without parameter updates. The best, second-best, and third-best results for each metric are highlighted in \textcolor[HTML]{FE0000}{red}, \textcolor[HTML]{3531FF}{blue}, and \textbf{bold}, respectively.}
    \label{main_zeroshot}
    \vspace{-16pt}
\end{table}
\subsection{Mechanism analysis}

\begin{figure*}[!ht]
\small{
  \centering
  \includegraphics[width=0.90\linewidth]{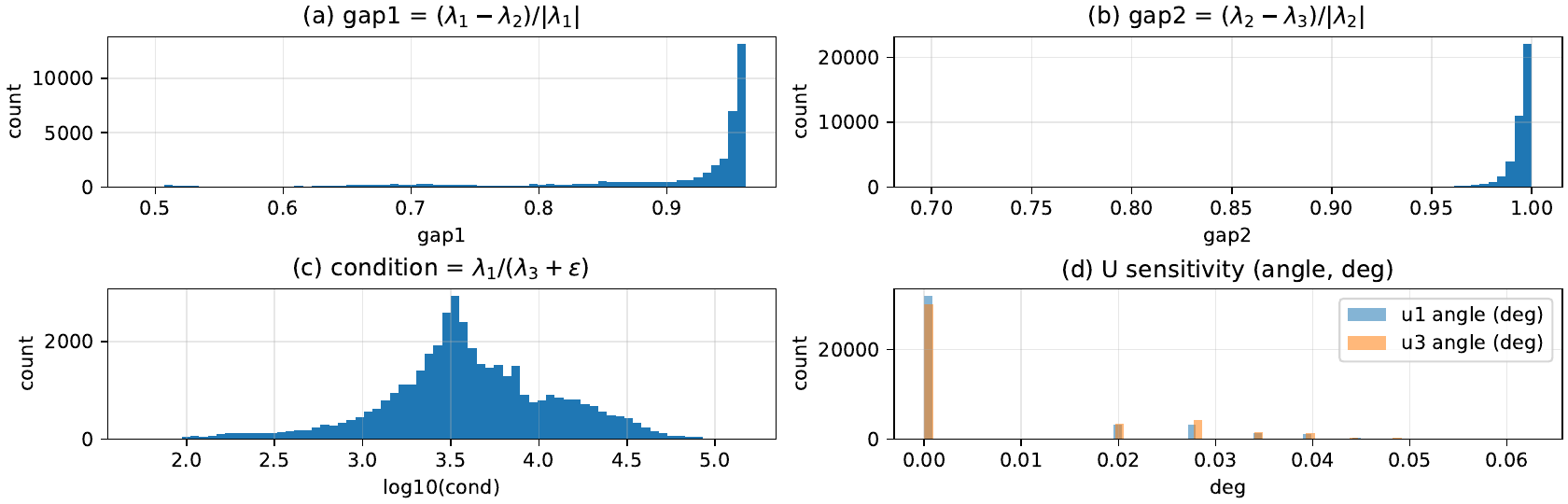}
  \caption{Spectrum and stability statistics of the aggregated Gram matrix $G$ over test windows.}
  \label{gram_stats}
 \vspace{-18pt}}
\end{figure*}

\paragraph{Ablation Studies} 

We ablate four components of NavFormer under both input lengths. Table~\ref{main_abl} reports the relative MAE increase for each ablation, defined as
\begin{equation}
    \Delta_{\%} = 100 \times (\mathrm{MAE}_{\mathrm{abl}} - \mathrm{MAE}_{\mathrm{full}}) / \mathrm{MAE}_{\mathrm{full}}.
\end{equation}
Removing harmonic tokens increases MAE on every flight. With $L=30$, the increase ranges from $3.81$ to $5.85\%$ with a mean of $5.05\%$. With $L=60$, the mean increase is $3.14\%$. This supports that explicit temporal harmonics compensate for limited context.Removing the canonical SPD modulation also degrades performance. With $L=30$, the mean increase is $4.03\%$ and the increase reaches $6.21\%$ on NV6. With $L=60$, the mean increase is $2.67\%$, but NV7 still shows a $5.88\%$ increase. This supports that geometry aware modulation matters most when orientation variability is high.

\begin{table}[!htbp]
    \centering
{\fontsize{6.5}{8.0}\selectfont
\begin{tabular}{lcccc}
\hline
        & \multicolumn{4}{c}{\textbf{Relative Performance Drop ($\%$) $\downarrow$}} \\ \cline{2-5} 
Methods & w/o Harmonic         & w/o SPD        & w/o FiLM        & Only FiLM        \\ \hline
\multicolumn{5}{l}{\textit{\textbf{Panel A. Sequence Length $L=30$}}}                \\ \hline
NV2     & 4.85                 & 3.14           & 2.20            & 2.15             \\
NV3     & 3.81                 & 5.15           & 0.46            & 3.89             \\
NV4     & 5.85                 & 1.83           & 2.54            & 1.75             \\
NV6     & 4.97                 & 6.21           & 6.62            & 2.77             \\
NV7     & 5.76                 & 3.82           & 0.99            & 4.13             \\ \hline
\multicolumn{5}{l}{\textit{\textbf{Panel B. Sequence Length $L=60$}}}                \\ \hline
NV2     & 2.97                 & 0.72           & 1.65            & 2.38             \\
NV3     & 3.96                 & 0.43           & 1.25            & 3.27             \\
NV4     & 2.72                 & 3.55           & 3.27            & 3.67             \\
NV6     & 2.67                 & 2.75           & 4.81            & 5.87             \\
NV7     & 3.40                 & 5.88           & 10.85           & 6.55             \\ \hline
\end{tabular}}
        \caption{Ablation study for NavFormer. Panel A uses input length $L=30$ and Panel B uses input length $L=60$. We report the relative MAE increase in \% with respect to the full model. }
    \label{main_abl}
    \vspace{-18pt}
\end{table}
Removing FiLM has a smaller effect for $L=30$, with a mean increase of $2.56\%$, but it has the largest effect for $L=60$, with a mean increase of $4.37\%$. The increase reaches $10.85\%$ on NV7. Longer context improves the state summary and strengthens the effect of state conditioned token modulation. The Only FiLM variant, which keeps FiLM while removing harmonic tokens and canonical SPD modulation, also underperforms the full model. Its mean increase is $2.94\%$ for $L=30$ and $4.35\%$ for $L=60$. These results show that FiLM alone cannot replace the combination of temporal priors and canonical SPD modulation.

\paragraph{Stability of the canonical frame} 
We examine when the eigenframe of the aggregated Gram matrix $G$ is identifiable and stable over the test distribution.
For each input window we compute eigenvalues $\lambda_{1} \ge \lambda_{2} \ge \lambda_{3}$ and the relative gaps ~\ref{eq:eigen}.
We also compute the condition proxy $\kappa = \lambda_{1}/(\lambda_{3} + \epsilon)$ for a small $\epsilon$. To test stability, we add a small perturbation to the triads, recompute the eigenvectors, and measure the angle between each original eigenvector $u_{i}$ and its perturbed counterpart $\tilde{u}_{i}$. Figure~\ref{gram_stats} shows that both gap statistics concentrate near $1$. This implies strong spectral separation in most windows, which makes the canonical basis close to unique up to sign \citep{davis1970rotation, yu2015useful}. The distribution of $\log_{10}\kappa$ is broad and often large, which indicates anisotropic energy across canonical directions and supports direction wise SPD scaling.

The eigenvector perturbation angles for $u_{1}$ and $u_{3}$ concentrate near $0^\circ$ and remain below about $0.06^\circ$ on almost all windows. This stability implies that small sensor noise does not materially change the canonical axes, so the learned SPD scaling admits a consistent interpretation as axis aligned gain control. A small subset of windows sits in the lower tail of the gap distribution. In those cases the eigenframe can drift within a nearly degenerate subspace. NavFormer retains a stable conditioning signal through the invariant scalar pathway, which does not depend on the eigenframe.

\paragraph{Show Case Analysis} We visualize a representative forecast on NV2 with lookback $L=60$ and horizon $L_{\mathrm{pred}}=120$ in Figure~\ref{figure_showcase}. NavFormer follows the ground truth trajectory through the steep rise near $t \approx 130$ and the subsequent decline, and it preserves the amplitude during the final increase. PatchTST and PAttn capture the global trend but smooth the peak and show a small lag around turning points. iTransformer and DLinear underestimate the mid horizon peak and the late horizon rise. TimesNet produces oscillatory deviations and overshoots, while LTC-CFC exhibits a sustained positive bias over much of the segment. This example aligns with Table~\ref{main_results} and supports the role of the geometric front end in reducing frame-dependent variation.

\begin{figure}[t]
\small{
  \centering
  \includegraphics[width=1\linewidth]{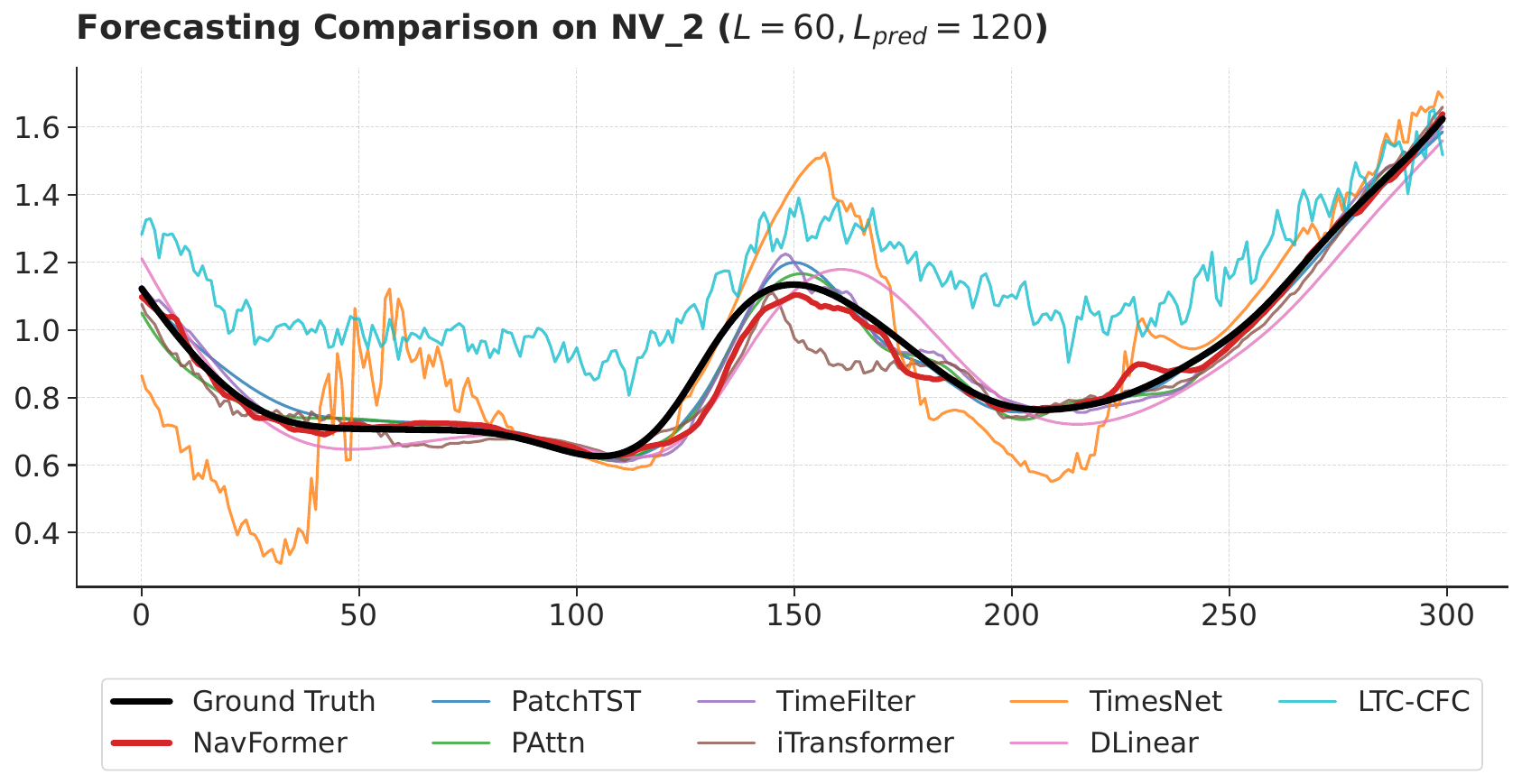}
  \caption{Forecasting comparison on NV2 with lookback $L=60$ and prediction horizon $L_{\text{pred}}=120$.}
  \label{figure_showcase}
 \vspace{-12pt}}
\end{figure}

\section{Conclusion}
NavFormer targets forecasting with rotating vector sensors by aligning window statistics instead of enforcing strict input invariance. The model combines rotation invariant scalar features with a sign invariant canonical symmetric positive definite spectral transform of triad vectors. Experiments on five aircraft flights show the best mean absolute error on most flights and consistent gains over strong transformer baselines. Few shot and zero shot tests show that the geometric front end improves sample efficiency and cross flight transfer.
More broadly, the proposed framework is compatible with sensing pipelines~\cite{lee2026sensor, lee2024dualpd}.
\clearpage

\section{Impact Statement}
This paper studies learning methods for magnetic navigation. The method is intended for benign navigation applications and should be evaluated carefully before deployment in safety critical systems.

\bibliography{example_paper}
\bibliographystyle{icml2026}

\newpage
\appendix
\onecolumn
\section{Dataset Description} \label{app:dataset}
\begin{wraptable}{r}{0.5\textwidth}
    \centering
    \centering
{\fontsize{8.0}{8.5}\selectfont
\begin{tabular}{ccccc}
\hline
Dataset & dim & \begin{tabular}[c]{@{}c@{}}Sequence \\ Length\end{tabular} & \begin{tabular}[c]{@{}c@{}}Prediction\\ Length\end{tabular} & Dataset Size           \\ \hline
NV2     & 26  & 30, 60\                                                 & 60, 120\                                                 & (124548, 41516, 41516) \\
NV3     & 26  & 30, 60\                                                 & 60, 120\                                                 & (96018 , 32006, 32006) \\
NV4     & 26  & 30, 60\                                                 & 60, 120\                                                 & (48844 , 16282, 16282) \\
NV6     & 26  & 30, 60\                                                 & 60, 120\                                                 & (64990 , 21664, 21664) \\
NV7     & 26  & 30, 60\                                                 & 60, 120\                                                 & (68703 , 22901, 22902) \\ \hline
\end{tabular}}
    \caption{Summary of the MagNav flights used in our experiments}
    \label{appendix_dataset}
    \vspace{-12pt} 
\end{wraptable}

We use the public MagNav challenge dataset collected near Ottawa, Ontario, Canada, through a collaboration between the Department of the Air Force MIT AI Accelerator and Sanders Geophysics Ltd. The platform is a Cessna 208B Grand Caravan instrumented with multiple magnetometers and standard flight telemetry. The magnetic suite contains five scalar magnetometers that measure total field intensity and three vector magnetometers that measure tri axial field components. The scalar sensors are optically pumped units, and the vector sensors are fluxgate units. The sensors are installed at distinct locations across the airframe, which creates meaningful variation in interference strength and structure. One scalar sensor is mounted on a boom behind the cabin, which reduces coupling to onboard magnetic sources and serves as the reference signal for supervision. The remaining scalar sensors are placed inside the cabin at different longitudinal positions and heights, which makes them strongly affected by airframe fields and electrical activity.

\begin{figure}[ht]
    \centering
    \includegraphics[width=0.60\linewidth]{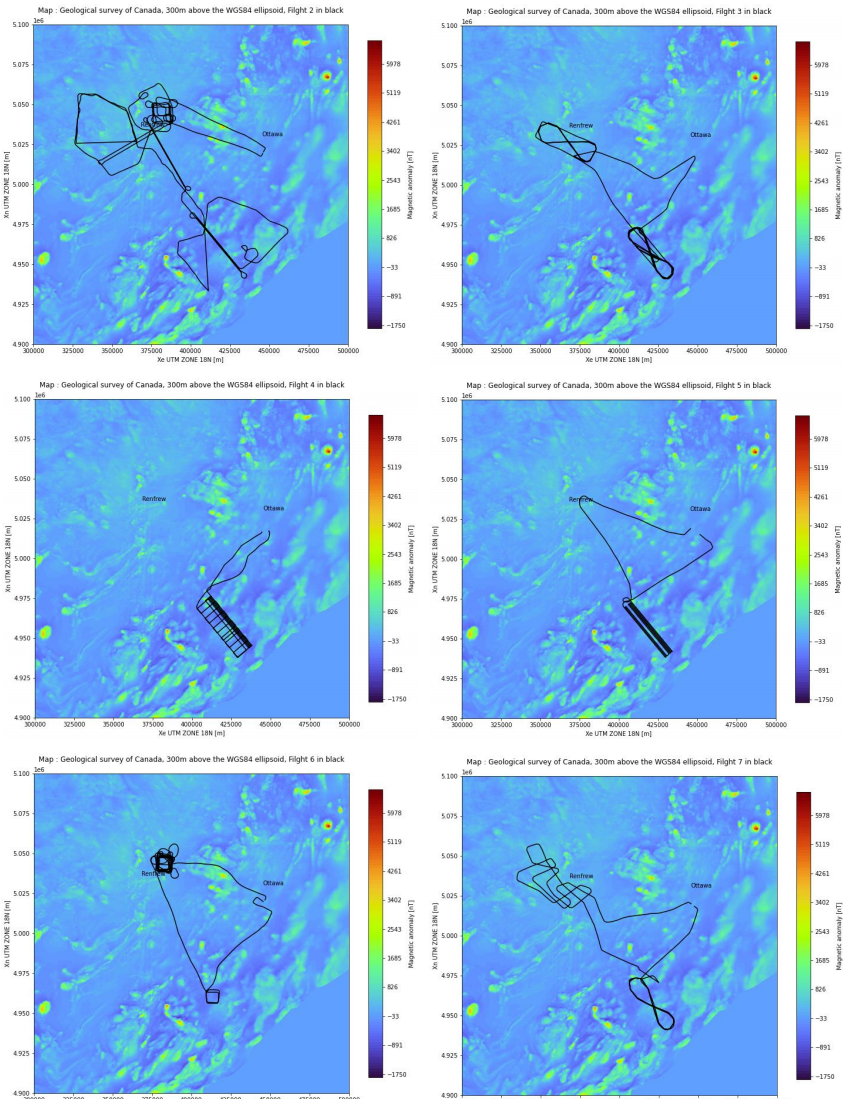}
    \caption{Flight path. Adapted from the MagNav repository (\url{https://github.com/Naatyu/MagNav}).}    \label{fig:flight_path}
\end{figure}

Table~\ref{appendix_dataset} summarizes the six flights, their segmentation, and their sample counts at 10 Hz. We partition the data to prevent spatial leakage where the aircraft revisits specific geographic locations. The training set comprises flights 1002 1004 and 1006 while flights 1003 and 1007 serve as unseen test cases. Flight 1003 features figure eight patterns and Flight 1007 includes complex maneuvering which challenges the generalization capabilities of the models. We apply the International Geomagnetic Reference Field model to remove the core field component from the raw readings. Ground stations provide corrections for diurnal temporal variations which further isolates the local crustal magnetic field. The total recording time exceeds 20 hours with flight altitudes ranging between 400 meters and 800 meters. The geographic region exhibits low crustal field variation which simplifies the isolation of aircraft induced noise for benchmarking purposes. All flights occurred during the day to ensure consistent environmental conditions. See flight path visualization Figure ~\ref{fig:flight_path}.

\section{Mathematical derivations} \label{appendix_math}
\begin{proposition}[Spectral Alignment of Canonical SPD] 
Let $G = U \Lambda U^\top$ be the aggregated Gram matrix and $M(\mathbf{s}) =U \Sigma(\mathbf{s}) U^\top$ the canonical SPD transform. Define the Gram matrix of the modulated triads as $\tilde{G}(\mathbf{s})\coloneqq\sum_{t=1}^L ( \tilde{B}_{t} \tilde{B}_{t}^\top + \tilde{C}_{t} \tilde{C}_{t}^\top + \tilde{D}_{t} \tilde{D}_{t}^\top )$.
Then $\tilde{G}(\mathbf{s}) = U \Sigma(\mathbf{s}) \Lambda \Sigma(\mathbf{s}) U^\top$. In particular, $\tilde{G}(\mathbf{s})$ has eigenvectors $U$ and eigenvalues$\sigma_i(\mathbf{s})^2 \lambda_i$ for $i=1,2,3$.\end{proposition}
\begin{proof}
Let $M \coloneqq M(\mathbf{s})$.
Since $\tilde B_t = M B_t$, we have
\begin{equation}
\tilde B_t \tilde B_t^\top = (M B_t)(M B_t)^\top = M (B_t B_t^\top) M^\top,    
\end{equation}
and similarly for $\tilde C_t, \tilde D_t$. Hence
\begin{equation}
\tilde G(\mathbf{s}) = \sum_{t=1}^L M (B_t B_t^\top + C_t C_t^\top + D_t D_t^\top) M^\top
= M G M^\top.
\end{equation}
Using $G=U\Lambda U^\top$ and $M=U\Sigma(\mathbf{s})U^\top$ with $U^\top U=I$ and $\Sigma$ diagonal, 
\begin{align}
\tilde G(\mathbf{s}) &= (U\Sigma U^\top)(U\Lambda U^\top)(U\Sigma U^\top) \\
&= U \Sigma \Lambda \Sigma U^\top.
\end{align}
Therefore $U^\top \tilde G(\mathbf{s}) U = \Sigma \Lambda \Sigma
= \mathrm{diag}(\sigma_1^2\lambda_1,\sigma_2^2\lambda_2,\sigma_3^2\lambda_3)$,
so $u_i$ is an eigenvector of $\tilde G(\mathbf{s})$ with eigenvalue $\sigma_i^2\lambda_i$.
\end{proof}

\section{Hyperparameter} \label{appendix:hyper}
We tuned hyperparameters via grid search to ensure a fair comparison across all baselines. The search space for general time series forecasting models covered a broad range of capacities. We varied the model dimension $d_{\text{model}}$ within the set $\{64, 128, 256, 512\}$. We selected feed-forward dimensions $d_{\text{ff}}$ from $\{64, 128, 256\}$. We tested encoder depths of 1, 2, 3, and 4 layers. The attention mechanism used either 2 or 4 heads. We adapted the search space for TimesNet to accommodate its specific structural requirements. We swept $d_{\text{model}}$ values of $\{8, 16, 32, 64\}$ and $d_{\text{ff}}$ values of $\{8, 16, 32\}$. We retained the same ranges for layers and heads. We additionally tuned the top-$k$ parameter from $\{3, 5, 7\}$. We chose the final configuration for each model based on the lowest validation error. See the hyperparmeter Table ~\ref{appendix_dataset}.

\begin{table}[!htbp]
    \centering
{\fontsize{8.0}{8.5}\selectfont
\begin{tabular}{ll}
\hline
\multicolumn{2}{l}{\textit{\textbf{Panel A. General Time Series Forecasting Models}}} \\
\textbf{Parameter}                          & \textbf{Values}                         \\ \hline
D\_MODELS                                   & 64, 128, 256, 512                            \\
D\_FF                                      & 64, 128, 256                           \\
E\_LAYERS\_LIST                             & 1, 2, 3, 4                                    \\
N\_HEADS\_LIST                              & 2, 4                                 \\ \hline
                                            &                                         \\ \hline
\multicolumn{2}{l}{\textit{\textbf{Panel B. TimesNet}}}                               \\
\textbf{Parameter}                          & \textbf{Values}                         \\ \hline
D\_MODELS                                   & 8, 16, 32, 64                               \\
D\_FF                                      & 8, 16, 32                               \\
E\_LAYERS\_LIST                             & 1, 2, 3, 4                                    \\
N\_HEADS\_LIST                              & 2, 4                                 \\
TOP\_K                                      & 3, 5, 7                                 \\ \hline
                                            &                                         \\ \hline

\end{tabular}}
        \caption{Hyperparameter grid search space}
    \label{appendix_dataset}
    \vspace{-12pt}
\end{table}

\clearpage

\section{Complete results for IGRF forecasting} \label{app:full_results}
\begin{table*}[!htbp]
    \centering
{\fontsize{6.0}{8.0}\selectfont
\begin{tabular}{cccccccccc}
\hline
\multicolumn{10}{l}{Panel A. Lookback window $L=30$ } \\
\multicolumn{2}{c}{Methods} & NavFormer & PatchTST & PAttn & TimeFilter & iTransformer & TimesNet & DLinear & LTC-CFC \\ \hline
\multicolumn{2}{c}{Metric} & MAE & MAE & MAE & MAE & MAE & MAE & MAE & MAE \\ \hline
 & \multicolumn{1}{c|}{60} & {\color[HTML]{3531FF} \textbf{0.0068 $\pm$ 0.0006}} & {\color[HTML]{FE0000} \textbf{0.0062 $\pm$ 0.0000}} & 0.0085 $\pm$ 0.0000 & \textbf{0.0082 $\pm$ 0.0007} & 0.0087 $\pm$ 0.0011 & 0.0116 $\pm$ 0.0009 & 0.0521 $\pm$ 0.0023 & 0.2780 $\pm$ 0.0190 \\
\multirow{-2}{*}{NV2} & \multicolumn{1}{c|}{120} & {\color[HTML]{FE0000} \textbf{0.0352 $\pm$ 0.0001}} & 0.0391 $\pm$ 0.0009 & {\color[HTML]{3531FF} \textbf{0.0366 $\pm$ 0.0004}} & \textbf{0.0368 $\pm$ 0.0016} & 0.0387 $\pm$ 0.0007 & 0.0457 $\pm$ 0.0005 & 0.0954 $\pm$ 0.0010 & 0.2987 $\pm$ 0.0361 \\
 & \multicolumn{1}{c|}{60} & {\color[HTML]{3531FF} \textbf{0.0094 $\pm$ 0.0005}} & {\color[HTML]{FE0000} \textbf{0.0088 $\pm$ 0.0001}} & \textbf{0.0107 $\pm$ 0.0001} & 0.0111 $\pm$ 0.0002 & 0.0118 $\pm$ 0.0001 & 0.0188 $\pm$ 0.0002 & 0.0713 $\pm$ 0.0040 & 0.2434 $\pm$ 0.0490 \\
\multirow{-2}{*}{NV3} & \multicolumn{1}{c|}{120} & {\color[HTML]{FE0000} \textbf{0.0607 $\pm$ 0.0011}} & {\color[HTML]{3531FF} \textbf{0.0609 $\pm$ 0.0004}} & \textbf{0.0630 $\pm$ 0.0008} & 0.0695 $\pm$ 0.0010 & 0.0643 $\pm$ 0.0004 & 0.0784 $\pm$ 0.0020 & 0.1537 $\pm$ 0.0023 & 0.2208 $\pm$ 0.0192 \\
 & \multicolumn{1}{c|}{60} & {\color[HTML]{FE0000} \textbf{0.0090 $\pm$ 0.0002}} & {\color[HTML]{3531FF} \textbf{0.0095 $\pm$ 0.0003}} & \textbf{0.0104 $\pm$ 0.0002} & 0.0107 $\pm$ 0.0005 & 0.0143 $\pm$ 0.0011 & 0.0248 $\pm$ 0.0015 & 0.0988 $\pm$ 0.0051 & 0.3485 $\pm$ 0.0667 \\
\multirow{-2}{*}{NV4} & \multicolumn{1}{c|}{120} & {\color[HTML]{FE0000} \textbf{0.0386 $\pm$ 0.0005}} & {\color[HTML]{3531FF} \textbf{0.0427 $\pm$ 0.0013}} & \textbf{0.0604 $\pm$ 0.0005} & 0.0770 $\pm$ 0.0005 & 0.0889 $\pm$ 0.0010 & 0.0967 $\pm$ 0.0015 & 0.1609 $\pm$ 0.0040 & 0.3528 $\pm$ 0.0470 \\
 & \multicolumn{1}{c|}{60} & {\color[HTML]{FE0000} \textbf{0.0339 $\pm$ 0.0012}} & \textbf{0.0393 $\pm$ 0.0016} & {\color[HTML]{3531FF} \textbf{0.0390 $\pm$ 0.0007}} & 0.0427 $\pm$ 0.0014 & 0.0487 $\pm$ 0.0010 & 0.0725 $\pm$ 0.0008 & 0.1941 $\pm$ 0.0051 & 0.3532 $\pm$ 0.0019 \\
\multirow{-2}{*}{NV6} & \multicolumn{1}{c|}{120} & {\color[HTML]{FE0000} \textbf{0.2094 $\pm$ 0.0081}} & {\color[HTML]{3531FF} \textbf{0.2279 $\pm$ 0.0162}} & 0.2571 $\pm$ 0.0267 & \textbf{0.2548 $\pm$ 0.0077} & 0.2697 $\pm$ 0.0003 & 0.3102 $\pm$ 0.0137 & 0.4350 $\pm$ 0.0021 & 0.3650 $\pm$ 0.0053 \\
 & \multicolumn{1}{c|}{60} & {\color[HTML]{3531FF} \textbf{0.0160 $\pm$ 0.0004}} & {\color[HTML]{FE0000} \textbf{0.0158 $\pm$ 0.0003}} & \textbf{0.0195 $\pm$ 0.0005} & 0.0221 $\pm$ 0.0002 & 0.0228 $\pm$ 0.0008 & 0.0278 $\pm$ 0.0014 & 0.0937 $\pm$ 0.0024 & 0.1281 $\pm$ 0.0027 \\
\multirow{-2}{*}{NV7} & \multicolumn{1}{c|}{120} & {\color[HTML]{FE0000} \textbf{0.0880 $\pm$ 0.0022}} & {\color[HTML]{3531FF} \textbf{0.0907 $\pm$ 0.0050}} & 0.0954 $\pm$ 0.0066 & 0.0993 $\pm$ 0.0037 & \textbf{0.0907 $\pm$ 0.0004} & 0.1063 $\pm$ 0.0025 & 0.1538 $\pm$ 0.0006 & 0.3666 $\pm$ 0.0908 \\ \hline
\multicolumn{1}{l}{} & \multicolumn{1}{l}{} & \multicolumn{1}{l}{} & \multicolumn{1}{l}{} & \multicolumn{1}{l}{} & \multicolumn{1}{l}{} & \multicolumn{1}{l}{} & \multicolumn{1}{l}{} & \multicolumn{1}{l}{} & \multicolumn{1}{l}{} \\ \hline
\multicolumn{10}{l}{Panel B. Lookback window $L=60$ } \\
\multicolumn{2}{c}{Methods} & NavFormer & PatchTST & PAttn & TimeFilter & iTransformer & TimesNet & DLinear & LTC-CFC \\ \hline
\multicolumn{2}{c}{Metric} & MAE & MAE & MAE & MAE & MAE & MAE & MAE & MAE \\ \hline
 & \multicolumn{1}{c|}{60} & {\color[HTML]{FE0000} \textbf{0.0061 $\pm$ 0.0002}} & 0.0085 $\pm$ 0.0003 & \textbf{0.0080 $\pm$ 0.0001} & 0.0089 $\pm$ 0.0001 & {\color[HTML]{3531FF} \textbf{0.0077 $\pm$ 0.0002}} & 0.0123 $\pm$ 0.0006 & 0.0372 $\pm$ 0.0031 & 0.2770 $\pm$ 0.0209 \\
\multirow{-2}{*}{NV2} & \multicolumn{1}{c|}{120} & {\color[HTML]{FE0000} \textbf{0.0369 $\pm$ 0.0001}} & 0.0419 $\pm$ 0.0005 & {\color[HTML]{3531FF} \textbf{0.0377 $\pm$ 0.0001}} & \textbf{0.0388 $\pm$ 0.0002} & 0.0389 $\pm$ 0.0001 & 0.0478 $\pm$ 0.0011 & 0.0843 $\pm$ 0.0023 & 0.2926 $\pm$ 0.0391 \\
 & \multicolumn{1}{c|}{60} & {\color[HTML]{3531FF} \textbf{0.0127 $\pm$ 0.0005}} & \textbf{0.0135 $\pm$ 0.0001} & {\color[HTML]{FE0000} \textbf{0.0127 $\pm$ 0.0001}} & 0.0156 $\pm$ 0.0007 & 0.0154 $\pm$ 0.0001 & 0.0234 $\pm$ 0.0017 & 0.0551 $\pm$ 0.0033 & 0.2256 $\pm$ 0.0215 \\
\multirow{-2}{*}{NV3} & \multicolumn{1}{c|}{120} & {\color[HTML]{3531FF} \textbf{0.0681 $\pm$ 0.0002}} & \textbf{0.0681 $\pm$ 0.0009} & 0.0692 $\pm$ 0.0015 & 0.0729 $\pm$ 0.0010 & {\color[HTML]{FE0000} \textbf{0.0671 $\pm$ 0.0006}} & 0.0809 $\pm$ 0.0006 & 0.1538 $\pm$ 0.0009 & 0.2206 $\pm$ 0.0191 \\
 & \multicolumn{1}{c|}{60} & {\color[HTML]{FE0000} \textbf{0.0106 $\pm$ 0.0005}} & \textbf{0.0120 $\pm$ 0.0004} & 0.0121 $\pm$ 0.0003 & {\color[HTML]{3531FF} \textbf{0.0115 $\pm$ 0.0004}} & 0.0145 $\pm$ 0.0016 & 0.0234 $\pm$ 0.0016 & 0.0788 $\pm$ 0.0013 & 0.3416 $\pm$ 0.0554 \\
\multirow{-2}{*}{NV4} & \multicolumn{1}{c|}{120} & {\color[HTML]{FE0000} \textbf{0.0397 $\pm$ 0.0012}} & {\color[HTML]{3531FF} \textbf{0.0454 $\pm$ 0.0012}} & 0.0533 $\pm$ 0.0044 & \textbf{0.0500 $\pm$ 0.0035} & 0.0779 $\pm$ 0.0007 & 0.0939 $\pm$ 0.0018 & 0.1499 $\pm$ 0.0012 & 0.3640 $\pm$ 0.0534 \\
 & \multicolumn{1}{c|}{60} & {\color[HTML]{FE0000} \textbf{0.0394 $\pm$ 0.0021}} & \textbf{0.0487 $\pm$ 0.0018} & {\color[HTML]{3531FF} \textbf{0.0462 $\pm$ 0.0014}} & 0.0515 $\pm$ 0.0014 & 0.0572 $\pm$ 0.0017 & 0.0942 $\pm$ 0.0018 & 0.2289 $\pm$ 0.0048 & 0.3545 $\pm$ 0.0035 \\
\multirow{-2}{*}{NV6} & \multicolumn{1}{c|}{120} & \textbf{0.2336 $\pm$ 0.0023} & {\color[HTML]{FE0000} \textbf{0.2246 $\pm$ 0.0068}} & {\color[HTML]{3531FF} \textbf{0.2333 $\pm$ 0.0083}} & 0.2385 $\pm$ 0.0102 & 0.2632 $\pm$ 0.0019 & 0.3129 $\pm$ 0.0055 & 0.4814 $\pm$ 0.0096 & 0.3636 $\pm$ 0.0043 \\
 & \multicolumn{1}{c|}{60} & {\color[HTML]{FE0000} \textbf{0.0186 $\pm$ 0.0007}} & {\color[HTML]{3531FF} \textbf{0.0202 $\pm$ 0.0002}} & 0.0210 $\pm$ 0.0004 & 0.0238 $\pm$ 0.0012 & 0.0247 $\pm$ 0.0005 & 0.0408 $\pm$ 0.0008 & 0.0711 $\pm$ 0.0024 & 0.3497 $\pm$ 0.0058 \\
\multirow{-2}{*}{NV7} & \multicolumn{1}{c|}{120} & {\color[HTML]{3531FF} \textbf{0.0971 $\pm$ 0.0011}} & {\color[HTML]{FE0000} \textbf{0.0947 $\pm$ 0.0013}} & 0.1056 $\pm$ 0.0020 & 0.0985 $\pm$ 0.0027 & \textbf{0.0974 $\pm$ 0.0009} & 0.1156 $\pm$ 0.0009 & 0.1383 $\pm$ 0.0006 & 0.3633 $\pm$ 0.0089 \\ \hline
\end{tabular}}
        \caption{Forecasting performance comparison. We report the Mean Absolute Error (MAE) averaged over prediction lengths $T \in \{60, 120\}$. \textbf{Panel A} presents results with a lookback window of $L=30$, while \textbf{Panel B} uses a lookback window of $L=60$. The best results are highlighted in \textcolor[HTML]{FE0000}{red}, the second-best in \textcolor[HTML]{3531FF}{blue} and third-best in \textbf{bold}, respectively.}
    \label{appendix_results_mae}
    \vspace{-12pt}
\end{table*}

\begin{table*}[!htbp]
    \centering
{\fontsize{6.0}{8.0}\selectfont
\begin{tabular}{cccccccccc}
\hline
\multicolumn{10}{l}{Panel A. Lookback window $L=30$ } \\
\multicolumn{2}{c}{Methods} & NavFormer & PatchTST & PAttn & TimeFilter & iTransformer & TimesNet & DLinear & LTC-CFC \\ \hline
\multicolumn{2}{c}{Metric} & RMSE & RMSE & RMSE & RMSE & RMSE & RMSE & RMSE & RMSE \\ \hline
 & \multicolumn{1}{c|}{60} & {\color[HTML]{FE0000} \textbf{0.0235 $\pm$ 0.0006}} & {\color[HTML]{3531FF} \textbf{0.0249 $\pm$ 0.0003}} & \textbf{0.0253 $\pm$ 0.0002} & 0.0602 $\pm$ 0.0008 & 0.0283 $\pm$ 0.0016 & 0.0334 $\pm$ 0.0013 & 0.0837 $\pm$ 0.0060 & 0.3699 $\pm$ 0.0417 \\
\multirow{-2}{*}{NV2} & \multicolumn{1}{c|}{120} & {\color[HTML]{FE0000} \textbf{0.0955 $\pm$ 0.0012}} & \textbf{0.1015 $\pm$ 0.0003} & 0.1019 $\pm$ 0.0004 & {\color[HTML]{3531FF} \textbf{0.1011 $\pm$ 0.0002}} & 0.1039 $\pm$ 0.0012 & 0.1160 $\pm$ 0.0012 & 0.1551 $\pm$ 0.0014 & 0.3839 $\pm$ 0.0429 \\
 & \multicolumn{1}{c|}{60} & {\color[HTML]{3531FF} \textbf{0.0263 $\pm$ 0.0017}} & {\color[HTML]{FE0000} \textbf{0.0252 $\pm$ 0.0001}} & \textbf{0.0271 $\pm$ 0.0004} & 0.0365 $\pm$ 0.0010 & 0.0607 $\pm$ 0.0005 & 0.0409 $\pm$ 0.0001 & 0.1034 $\pm$ 0.0034 & 0.2995 $\pm$ 0.0465 \\
\multirow{-2}{*}{NV3} & \multicolumn{1}{c|}{120} & {\color[HTML]{3531FF} \textbf{0.1342 $\pm$ 0.0025}} & {\color[HTML]{FE0000} \textbf{0.1333 $\pm$ 0.0007}} & 0.1362 $\pm$ 0.0024 & 0.1565 $\pm$ 0.0014 & \textbf{0.1344 $\pm$ 0.0003} & 0.1478 $\pm$ 0.0018 & 0.2241 $\pm$ 0.0034 & 0.2954 $\pm$ 0.0175 \\
 & \multicolumn{1}{c|}{60} & {\color[HTML]{FE0000} \textbf{0.0474 $\pm$ 0.0002}} & {\color[HTML]{3531FF} \textbf{0.0482 $\pm$ 0.0001}} & \textbf{0.0498 $\pm$ 0.0002} & 0.0515 $\pm$ 0.0010 & \textbf{0.0498 $\pm$ 0.0011} & 0.0663 $\pm$ 0.0011 & 0.1368 $\pm$ 0.0041 & 0.4409 $\pm$ 0.0463 \\
\multirow{-2}{*}{NV4} & \multicolumn{1}{c|}{120} & {\color[HTML]{FE0000} \textbf{0.1075 $\pm$ 0.0012}} & {\color[HTML]{3531FF} \textbf{0.1115 $\pm$ 0.0020}} & \textbf{0.1318 $\pm$ 0.0011} & 0.1637 $\pm$ 0.0024 & 0.1623 $\pm$ 0.0015 & 0.1764 $\pm$ 0.0019 & 0.2282 $\pm$ 0.0060 & 0.4448 $\pm$ 0.0375 \\
 & \multicolumn{1}{c|}{60} & {\color[HTML]{FE0000} \textbf{0.0905 $\pm$ 0.0019}} & {\color[HTML]{3531FF} \textbf{0.0970 $\pm$ 0.0029}} & \textbf{0.1005 $\pm$ 0.0013} & 0.1117 $\pm$ 0.0012 & 0.1118 $\pm$ 0.0016 & 0.1466 $\pm$ 0.0027 & 0.2942 $\pm$ 0.0043 & 0.4343 $\pm$ 0.0024 \\
\multirow{-2}{*}{NV6} & \multicolumn{1}{c|}{120} & {\color[HTML]{FE0000} \textbf{0.4461 $\pm$ 0.0160}} & {\color[HTML]{3531FF} \textbf{0.4504 $\pm$ 0.0238}} & 0.4895 $\pm$ 0.0264 & \textbf{0.4823 $\pm$ 0.0051} & 0.5004 $\pm$ 0.0011 & 0.5517 $\pm$ 0.0159 & 0.6801 $\pm$ 0.0024 & 0.4460 $\pm$ 0.0031 \\
 & \multicolumn{1}{c|}{60} & {\color[HTML]{3531FF} \textbf{0.0361 $\pm$ 0.0009}} & {\color[HTML]{FE0000} \textbf{0.0355 $\pm$ 0.0008}} & \textbf{0.0424 $\pm$ 0.0011} & 0.0476 $\pm$ 0.0005 & 0.0470 $\pm$ 0.0014 & 0.0518 $\pm$ 0.0014 & 0.1281 $\pm$ 0.0027 & 0.3385 $\pm$ 0.0155 \\
\multirow{-2}{*}{NV7} & \multicolumn{1}{c|}{120} & {\color[HTML]{3531FF} \textbf{0.1648 $\pm$ 0.0037}} & 0.1725 $\pm$ 0.0033 & 0.1716 $\pm$ 0.0020 & 0.1887 $\pm$ 0.0058 & {\color[HTML]{FE0000} \textbf{0.1579 $\pm$ 0.0006}} & 0.1752 $\pm$ 0.0019 & 0.2174 $\pm$ 0.0015 & 0.4444 $\pm$ 0.0925 \\ \hline
\multicolumn{1}{l}{} & \multicolumn{1}{l}{} & \multicolumn{1}{l}{} & \multicolumn{1}{l}{} & \multicolumn{1}{l}{} & \multicolumn{1}{l}{} & \multicolumn{1}{l}{} & \multicolumn{1}{l}{} & \multicolumn{1}{l}{} & \multicolumn{1}{l}{} \\ \hline
\multicolumn{10}{l}{Panel B. Lookback window $L=60$ } \\
\multicolumn{2}{c}{Methods} & NavFormer & PatchTST & PAttn & TimeFilter & iTransformer & TimesNet & DLinear & LTC-CFC \\ \hline
\multicolumn{2}{c}{Metric} & RMSE & RMSE & RMSE & RMSE & RMSE & RMSE & RMSE & RMSE \\ \hline
 & \multicolumn{1}{c|}{60} & {\color[HTML]{FE0000} \textbf{0.0241 $\pm$ 0.0006}} & \textbf{0.0280 $\pm$ 0.0011} & {\color[HTML]{3531FF} \textbf{0.0269 $\pm$ 0.0003}} & 0.0321 $\pm$ 0.0007 & 0.0297 $\pm$ 0.0002 & 0.0373 $\pm$ 0.0009 & 0.0604 $\pm$ 0.0038 & 0.3674 $\pm$ 0.0439 \\
\multirow{-2}{*}{NV2} & \multicolumn{1}{c|}{120} & {\color[HTML]{3531FF} \textbf{0.1032 $\pm$ 0.0001}} & \textbf{0.1037 $\pm$ 0.0002} & {\color[HTML]{FE0000} \textbf{0.1024 $\pm$ 0.0003}} & 0.1044 $\pm$ 0.0004 & 0.1065 $\pm$ 0.0003 & 0.1173 $\pm$ 0.0019 & 0.1446 $\pm$ 0.0014 & 0.3893 $\pm$ 0.0581 \\
 & \multicolumn{1}{c|}{60} & \textbf{0.0326 $\pm$ 0.0008} & {\color[HTML]{FE0000} \textbf{0.0604 $\pm$ 0.0002}} & {\color[HTML]{3531FF} \textbf{0.0323 $\pm$ 0.0002}} & 0.0404 $\pm$ 0.0014 & 0.0371 $\pm$ 0.0004 & 0.0511 $\pm$ 0.0033 & 0.0803 $\pm$ 0.0031 & 0.2861 $\pm$ 0.0361 \\
\multirow{-2}{*}{NV3} & \multicolumn{1}{c|}{120} & \textbf{0.1506 $\pm$ 0.0005} & {\color[HTML]{3531FF} \textbf{0.1435 $\pm$ 0.0013}} & 0.1508 $\pm$ 0.0051 & 0.1587 $\pm$ 0.0031 & {\color[HTML]{FE0000} \textbf{0.1416 $\pm$ 0.0007}} & 0.1531 $\pm$ 0.0010 & 0.2285 $\pm$ 0.0014 & 0.2950 $\pm$ 0.0174 \\
 & \multicolumn{1}{c|}{60} & 0.0550 $\pm$ 0.0016 & {\color[HTML]{FE0000} \textbf{0.0492 $\pm$ 0.0006}} & {\color[HTML]{3531FF} \textbf{0.0509 $\pm$ 0.0002}} & 0.0524 $\pm$ 0.0003 & 0.0515 $\pm$ 0.0014 & 0.0687 $\pm$ 0.0055 & 0.1135 $\pm$ 0.0003 & 0.4321 $\pm$ 0.0443 \\
\multirow{-2}{*}{NV4} & \multicolumn{1}{c|}{120} & {\color[HTML]{3531FF} \textbf{0.1122 $\pm$ 0.0024}} & {\color[HTML]{FE0000} \textbf{0.1117 $\pm$ 0.0016}} & \textbf{0.1253 $\pm$ 0.0055} & 0.1257 $\pm$ 0.0047 & 0.1515 $\pm$ 0.0015 & 0.1764 $\pm$ 0.0020 & 0.2201 $\pm$ 0.0015 & 0.4571 $\pm$ 0.0421 \\
 & \multicolumn{1}{c|}{60} & {\color[HTML]{FE0000} \textbf{0.0988 $\pm$ 0.0036}} & {\color[HTML]{3531FF} \textbf{0.1112 $\pm$ 0.0021}} & \textbf{0.1126 $\pm$ 0.0027} & 0.1232 $\pm$ 0.0024 & 0.1268 $\pm$ 0.0024 & 0.1901 $\pm$ 0.0010 & 0.3477 $\pm$ 0.0050 & 0.4355 $\pm$ 0.0035 \\
\multirow{-2}{*}{NV6} & \multicolumn{1}{c|}{120} & 0.4790 $\pm$ 0.0029 & {\color[HTML]{FE0000} \textbf{0.4502 $\pm$ 0.0055}} & \textbf{0.4770 $\pm$ 0.0134} & {\color[HTML]{3531FF} \textbf{0.4690 $\pm$ 0.0077}} & 0.4992 $\pm$ 0.0027 & 0.5604 $\pm$ 0.0050 & 0.7332 $\pm$ 0.0095 & 0.4411 $\pm$ 0.0025 \\
 & \multicolumn{1}{c|}{60} & {\color[HTML]{3531FF} \textbf{0.0421 $\pm$ 0.0011}} & {\color[HTML]{FE0000} \textbf{0.0407 $\pm$ 0.0005}} & \textbf{0.0452 $\pm$ 0.0009} & 0.0499 $\pm$ 0.0022 & 0.0489 $\pm$ 0.0006 & 0.0738 $\pm$ 0.0015 & 0.1027 $\pm$ 0.0016 & 0.4216 $\pm$ 0.0044 \\
\multirow{-2}{*}{NV7} & \multicolumn{1}{c|}{120} & {\color[HTML]{3531FF} \textbf{0.1802 $\pm$ 0.0010}} & 0.1836 $\pm$ 0.0021 & 0.1928 $\pm$ 0.0035 & \textbf{0.1833 $\pm$ 0.0046} & {\color[HTML]{FE0000} \textbf{0.1707 $\pm$ 0.0004}} & 0.1928 $\pm$ 0.0034 & 0.2047 $\pm$ 0.0005 & 0.4429 $\pm$ 0.0033 \\ \hline
\end{tabular}}
        \caption{Forecasting performance comparison. We report the Root Mean Squared Error (RMSE) averaged over prediction lengths $T \in \{60, 120\}$. \textbf{Panel A} presents results with a lookback window of $L=30$, while \textbf{Panel B} uses a lookback window of $L=60$. The best results are highlighted in \textcolor[HTML]{FE0000}{red}, the second-best in \textcolor[HTML]{3531FF}{blue} and third-best in \textbf{bold}, respectively.}
    \label{appendix_results_rmse}
    \vspace{-12pt}
\end{table*}

\end{document}